\documentclass{article}

\usepackage{microtype}
\usepackage{graphicx}
\usepackage{booktabs}
\usepackage{caption}
\usepackage{subcaption}
\usepackage{amsthm}
\usepackage{mathtools, amssymb}

\usepackage{hyperref}

\usepackage[accepted]{icml2021}

\icmltitlerunning{A Bayesian Approach to Invariant Deep Neural Networks}
\newtheorem{theorem}{Definition}
\newtheorem{lemma}[theorem]{Lemma}
\begin{document}

\twocolumn[
\icmltitle{A Bayesian Approach to Invariant Deep Neural Networks}



\icmlsetsymbol{equal}{*}

\begin{icmlauthorlist}
\icmlauthor{Nikolaos Mourdoukoutas}{eth}
\icmlauthor{Marco Federici}{ams}
\icmlauthor{Georges Pantalos}{eth}
\icmlauthor{Mark van der Wilk}{imp}
\icmlauthor{Vincent Fortuin}{eth}
\end{icmlauthorlist}

\icmlaffiliation{eth}{ETH Zürich, Zürich, Switzerland}
\icmlaffiliation{ams}{University of Amsterdam, Amsterdam, Netherlands}
\icmlaffiliation{imp}{Imperial College London, London, United Kingdom}

\icmlcorrespondingauthor{Nikolaos Mourdoukoutas}{nmourdou@ethz.ch}

\icmlkeywords{Machine Learning, ICML, Bayesian Neural Networks, Invariance, Group theoretic methods in Machine Learning}

\vskip 0.3in
]



\printAffiliationsAndNotice{}  

\begin{abstract}
We propose a novel Bayesian neural network architecture that can learn invariances from data alone by inferring a posterior distribution over different weight-sharing schemes. 
We show that our model outperforms other non-invariant architectures, when trained on datasets that contain specific invariances. The same holds true when no data augmentation is performed. 
\end{abstract}

\section{Introduction}
\label{introduction}
Deep learning models interpolate famously well on data that are generated from the training distribution. Nevertheless, when it comes to generalizing to out-of-distribution examples (e.g., transformed inputs), their predictive potential is more restricted.
For example, a classifier might be able to correctly predict the label of a handwritten digit, but could easily fail when the digit is rotated. 

A typical approach to solve this problem is to perform data augmentation (DA), where one includes transformed inputs in the training set and consequently boosts the performance of the learned model when presented with similar examples. However, this method is not guaranteed to be successful \cite{lyle2020benefits}.
This is no surprise, as learning to label an image that is rotated by $90^{\circ}$, for instance, does not imply generalization to different degrees of rotation.

A different class of methods emerges from the fact that for a linear map to be invariant under some transformation, there must be a specific weight-sharing scheme. This is a consequence of Schur's lemma \cite{pmlr-v80-kondor18a}. To that end, if one aims to be robust against a specific type of transformation, it is sufficient to know the corresponding way the parameters of the networks should be shared.

\subsection{Contributions}

We propose a method to learn such weight-sharing schemes from data. As a proof of concept, we focus on being invariant to two types of transformations applied on images, namely rotations and flips.
However, our algorithm can be applied to any other choice of symmetry, as long as the corresponding weight-sharing scheme is available.
Apart from achieving good performance during inference, our model is able to learn such invariances from data.
This is achieved by specifying a probability distribution over the weight-sharing schemes for the input layer, thus formulating a Bayesian neural network. We specify a prior over the parameters of this distribution and then perform MAP inference.
We empirically verify and illustrate the capabilities of our model on the MNIST dataset.

\section{Background}

In order to construct weight-sharing schemes, upon which we will later define a distribution, we will use the concept of the so-called Reynolds' operator, which has been used for similar reasons in previous work \citep{yarotsky2018universal, vanderpol2021mdp, mouli2021neural}. While our notation will be partly adapted from \citet{mouli2021neural}, we would like to highlight that our approach is quite different. We do not employ any causal mechanisms nor do we make any assumptions on the data-generating process, as our method does not entail any data augmentation.

Suppose we wish to perform classification of images of shape $3\times n \times n$.
We flatten the images and consider them as vectors in the input space $\mathcal{X}= \mathbb{R}^{3n^{2}}$.
A function $f: \mathcal{X}  \rightarrow \mathcal{X}$ is invariant to a linear transformation, represented by a matrix $A$, if the following equality holds for all $x \in \mathcal{X}$: $f(Ax)=f(x)$. In this case, we call $A$ a symmetry of $\mathcal{X}$. If we have another symmetry $B$, then it easily follows that $AB$ is a symmetry as well. This leads us to consider groups of transformations, with the function composition being the group multiplication.

\begin{theorem}
Let $G$ be a (finite) group of linear automorphisms from $\mathcal{X}$ to itself.
A transformation $\overline{T}$ is called $G$-invariant, if it holds that $\overline{T}(Tx)=\overline{T}(x)$, for all $x \in \mathcal{X}$ and $T \in G$.
\end{theorem}

An element of the group will act on  $\mathcal{X}$ via its matrix representation. Explicitly, we can express an arbitrary $T \in G$ with $T(x)=Ax$, for some $A \in \mathbb{R}^{3n^{2} \times 3n^{2}}$ and $x \in \mathcal{X}.$

We wish to specify a weight-sharing scheme, which will enable our neural network to generalize well when it is presented with examples that are generated by the aforementioned action on the input space. Formally, a layer with weights $w$, bias $b$, and activation function $\sigma$ is $G$-invariant if $\sigma\left(w^{\top} x+b\right)=\sigma\left(w^{\top} T x+b\right)$ for all $T \in G$.

This will be achieved through the Reynolds operator, which we now introduce.

\begin{theorem}
Let G be a (finite) group of linear automorphisms. The mapping $\overline{T}: x \mapsto \frac{1}{|G|} \sum_{T \in G} T(x)$ is the Reynolds operator of the group.
\end{theorem}

It can be shown that the Reynolds operator is $G$-invariant and a projection, that is, it holds that $\overline{T}^2=\overline{T}$.
An immediate consequence of the latter is that all its eigenvalues are either 0 or 1. This fact enables the characterization of $G$-invariant layers: a layer has this property if and only if its weights are in the span of the eigenvectors of the Reynolds operator with eigenvalue 1. Proofs of these claims can be found in Appendix~\ref{sec:algebra}.

To summarize, suppose we wish to be invariant to a particular group $G$.
We then proceed by computing the matrix representation of the group's Reynolds operator; computing the eigenvectors $(v_{i})_{i=1}^d$  that correspond to the eigenvalue 1; expressing every neuron in the first hidden layer of the network as $w = \sum_{i=1}^d a_i v_i$; and learning the parameters $(a_i)_{i=1}^d$ for every neuron.
We now introduce our method.

\section{Method}

Consider a collection of groups $\left (G_k\right)_{k=1}^q$ of linear automorphisms, under which we wish to be invariant. Let $\left(V_k\right)_{k=1}^q$ be the matrices that have as columns the eigenvectors which span the eigenspace of the eigenvalue 1 of the corresponding Reynolds operators, where $V_{k} \in \mathbb{R}^{d_k \times 3n^2}$ for $k=1,...,q$. For an illustration of how to compute these matrices, we refer to Appendix~\ref{sec:algebra}.

We aim to learn whether any of these invariances are present in the data (or could potentially be during testing). To this end, we consider a Categorical distribution over the possible symmetries. The corresponding probabilities are annotated by $\left( p_k \right)_{k=1}^{q+1}$, where $p_{q+1}$ is the probability that an observation is not indicative of any symmetry. Of course, it holds that $\sum_{k=1}^{q+1} p_{k}=1$. Let $\boldsymbol{\pi}=\left(p_{1}, ..., p_{q+1}\right)$.
We assume that $ \boldsymbol{\pi} \sim \operatorname{Dir}(\boldsymbol{\alpha})$, where $\boldsymbol{\alpha}=\left(\alpha_1, ..., \alpha_{p+1}\right)$ is the concentration parameter of the Dirichlet distribution and $\alpha _k>0$.  We note that if we are interested in only one group, that is, $q=1$, then we suppose that $\boldsymbol{\pi}=p_1 \sim \operatorname{Beta}(\alpha_1, \alpha_2)$. We will explain how we choose $\alpha$ in both cases later. By specifying a prior over the parameters of interest, we enable maximum a posteriori (MAP) estimation.

We proceed with the construction of our probabilistic invariant input layer. Fix $k \in \left \{1,...,q+1 \right \}$ and consider the mapping $F: \mathbb{R}^{3n^2} \rightarrow \mathbb{R}^m$, which maps the input to the first hidden representation. We construct its weight matrix as follows:
Every neuron in the hidden layer has the option to be invariant with respect to the group $G_k$, with probability $p_{k}$. 
To achieve this, a neuron has to lie in the respective eigenspace. Hence, we first concatenate the matrices $\left(V_k \right)_{k=1}^q$, which are the bases of these spaces, into $V = (p_{1}V_{1},..., p_{q}V_{q})^{\top} \in \mathbb{R}^{d \times 3n^2}$, where $d=\sum_{i=1}^q d_i$.  
Then, for every $j \in \{1,\dots,m\}$, we consider the learnable parameters $(a_{ji}^{k})_{i=1}^{d_k}$, or equivalently $a^j=((a_{ji}^{1})_{i=1}^{d_1},..., (a_{ji}^{q})_{i=1}^{d_q})$.
This results in a learnable matrix $A \in \mathbb{R}^{d\times m}$. Let $I \in \mathbb{R}^{d\times 3n^2}$ be a generalized identity matrix with ones on the main diagonal and zeros everywhere else.
The weight matrix for the first hidden layer is then $W=A^{\top}(V+I p_{q+1}) \in \mathbb{R}^{m \times 3n^2}$.

Essentially, we have defined a Categorical posterior distribution over the weights of our network. This construction allows the network to choose which kind of transformations it should be invariant to, according to the data it is presented. While we treat the probabilities $\boldsymbol{\pi}$ in a Bayesian way, all the other parameters can be treated deterministically. 

Our goal is to perform classification. Let $\mathcal{Y}=\{1,\dots,m\}$ be the output space, where $m$ is the number of classes in the dataset. Then, we wish to minimize the negative log-posterior over the training data points $(x_i, y_i)_{i=1}^n \subseteq \mathcal{X}\times \mathcal{Y}$, that is, we solve
\vspace{-5pt}$$
\arg \min _{\boldsymbol{\pi}, A, W}-\log p(\boldsymbol{\pi} \mid \boldsymbol{\alpha})-\sum_{i=1}^{n} \log p\left(y_{i} \mid {x}_{i}, \boldsymbol{\pi}, A, W\right),$$

\vspace{-5pt}
where $W$ are the learnable parameters for the rest of the layers and the likelihood of the labels is Categorical.
To estimate gradients for the parameters $\boldsymbol{\pi}$, we use the Gumbel-Softmax reparameterization trick \cite{jang2017categorical}. 

The regularization term of the objective function depends on the concentration hyperparameter $\boldsymbol{\alpha}$. It should be chosen to yield an uninformative prior, in order to be as objective as possible when it comes to learning any potential invariances from the data. In the simplest scenario where $q=1$, one can choose a Jeffreys prior \cite{10.2307/97883}, that is, $\left(\alpha_1, \alpha_2 \right)=\left( \frac{1}{2}, \frac{1}{2} \right)$. When $q \geq 2$, a uniform choice over the simplex with $\alpha_i \equiv \alpha$, for some $\alpha>0$ works in a similar way.

The second term tries to learn the right parameter-sharing scheme from the available candidates, while at the same time aiming to correctly classify all the training examples. 
We hypothesize that the model should be able to detect whether there are symmetries in our data in terms of the final MAP estimate of $\boldsymbol{\pi}$. Consequently, it should closely compete with its frequentist counterparts with the proper weight-sharing schemes in each case. 

On the other hand, when trained on non-image data, the estimator should converge to $ \boldsymbol{\hat{\pi}}_{MAP}=(0,..., 0, 1)$, that is, the network should prefer to have a regular multi-layer perceptron (MLP) network structure. 

\begin{figure*}[h]
     \centering
     
     \begin{subfigure}[b]{0.24\textwidth}
         \centering
         \includegraphics[width=\textwidth]{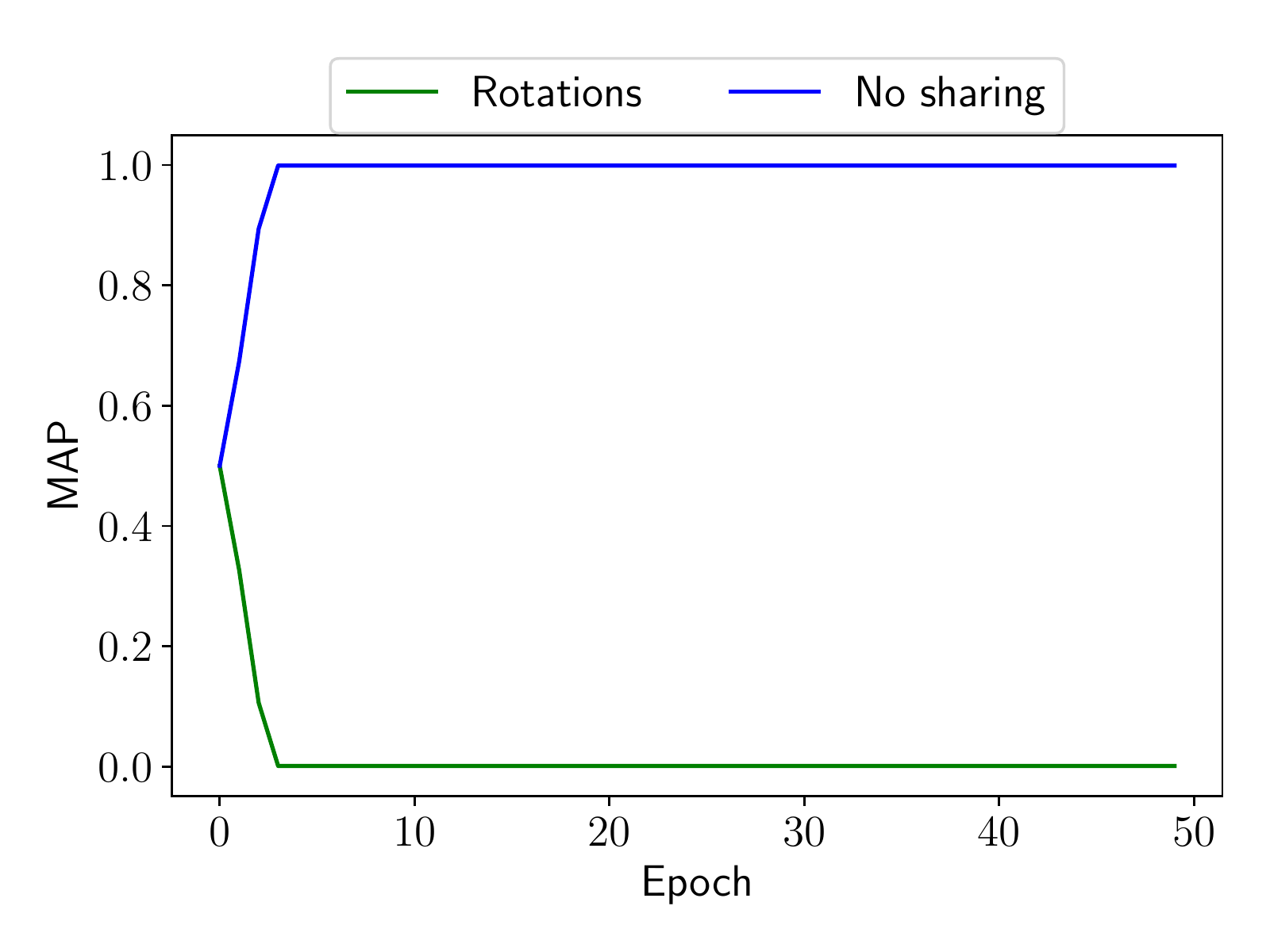}
         \caption{Permuted dataset}
         \label{fig:maps-permuted}
     \end{subfigure}
     \begin{subfigure}[b]{0.24\textwidth}
         \centering
         \includegraphics[width=\textwidth]{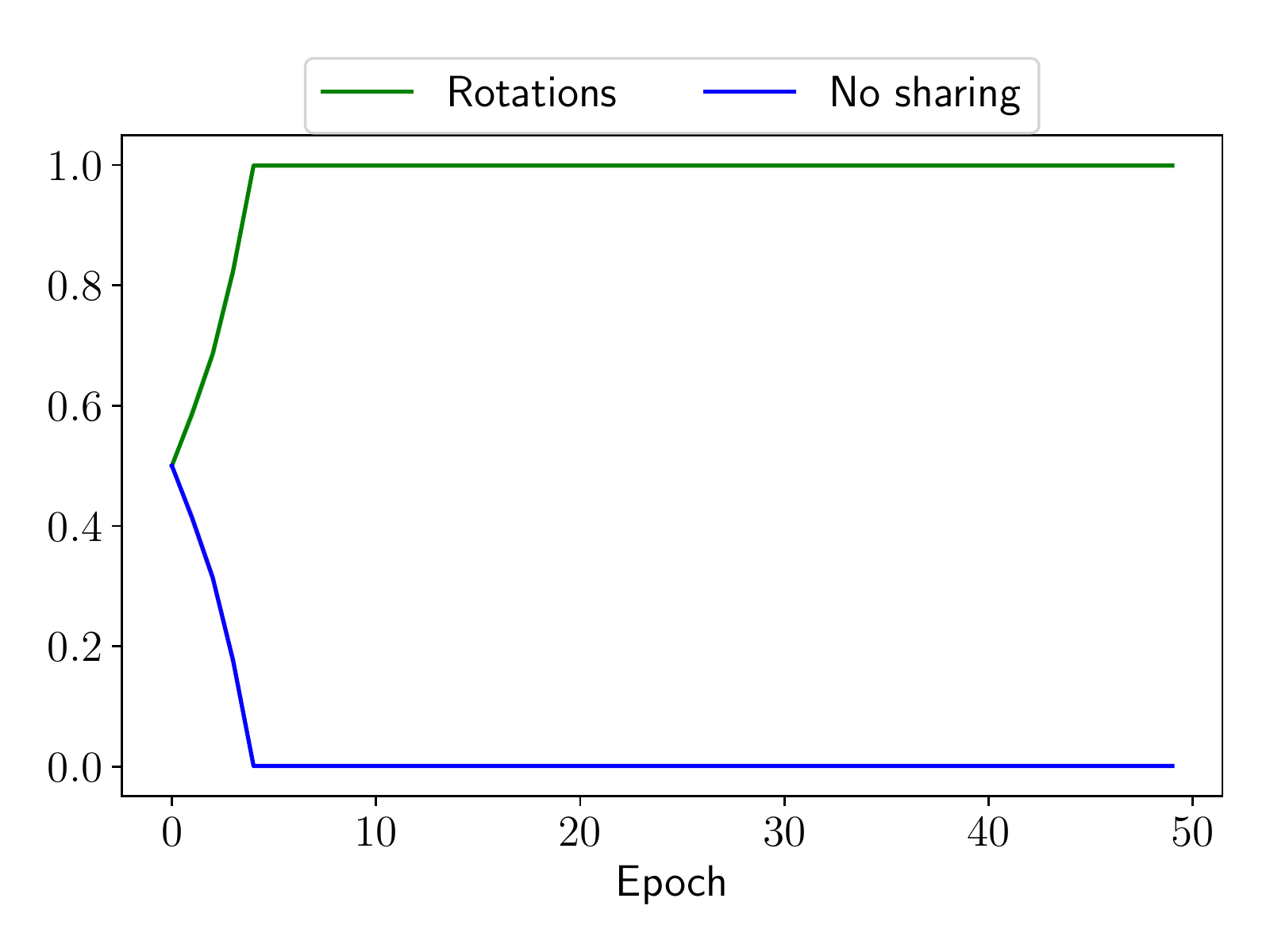}
         \caption{Rotated dataset}
         \label{fig:maps-rotated}
     \end{subfigure}
    \begin{subfigure}[b]{0.24\textwidth}
         \centering
         \includegraphics[width=\textwidth]{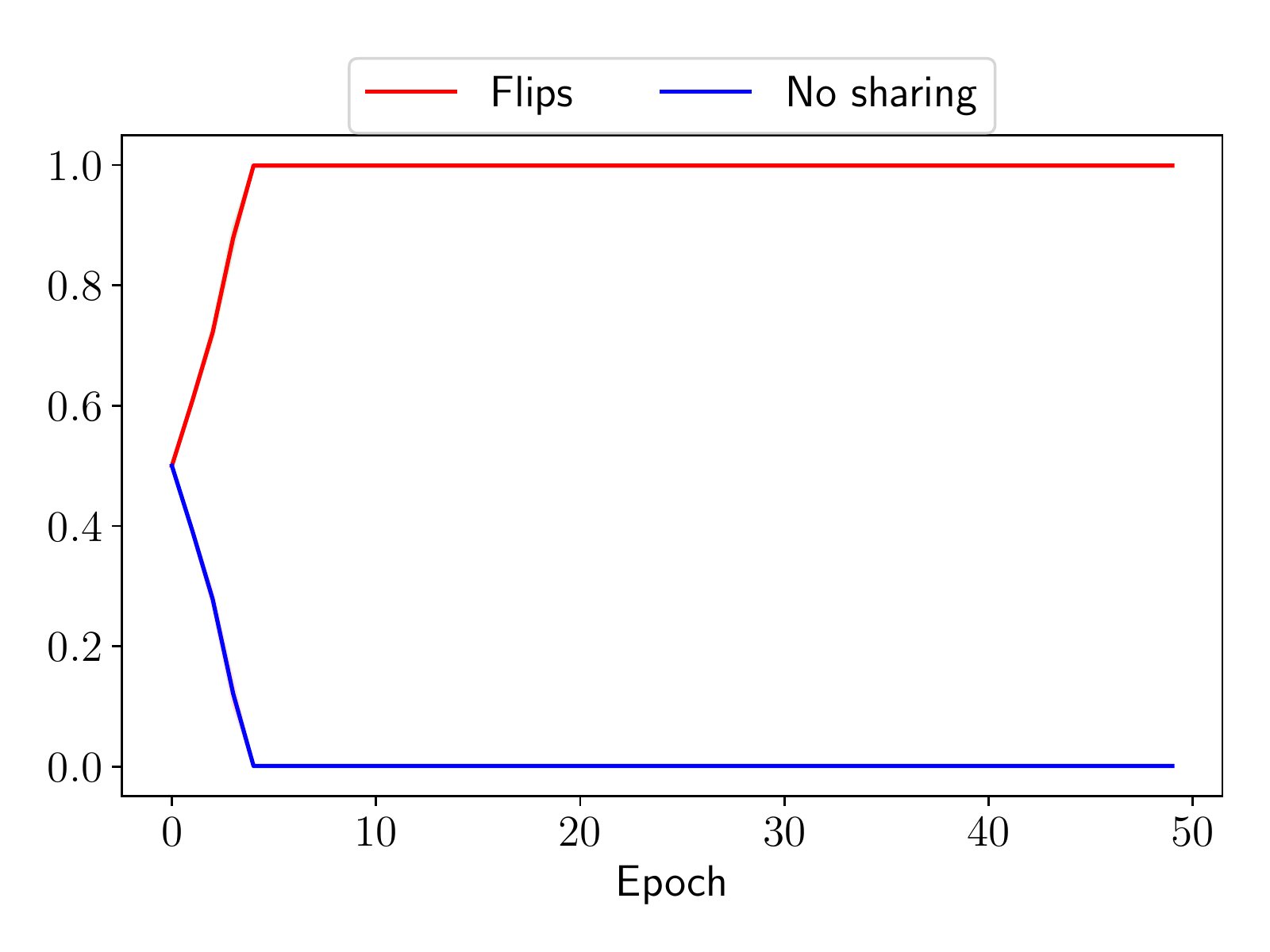}
         \caption{Flipped dataset}
         \label{fig:maps-flipped}
     \end{subfigure}
     \begin{subfigure}[b]{0.24\textwidth}
         \centering
         \includegraphics[width=\textwidth]{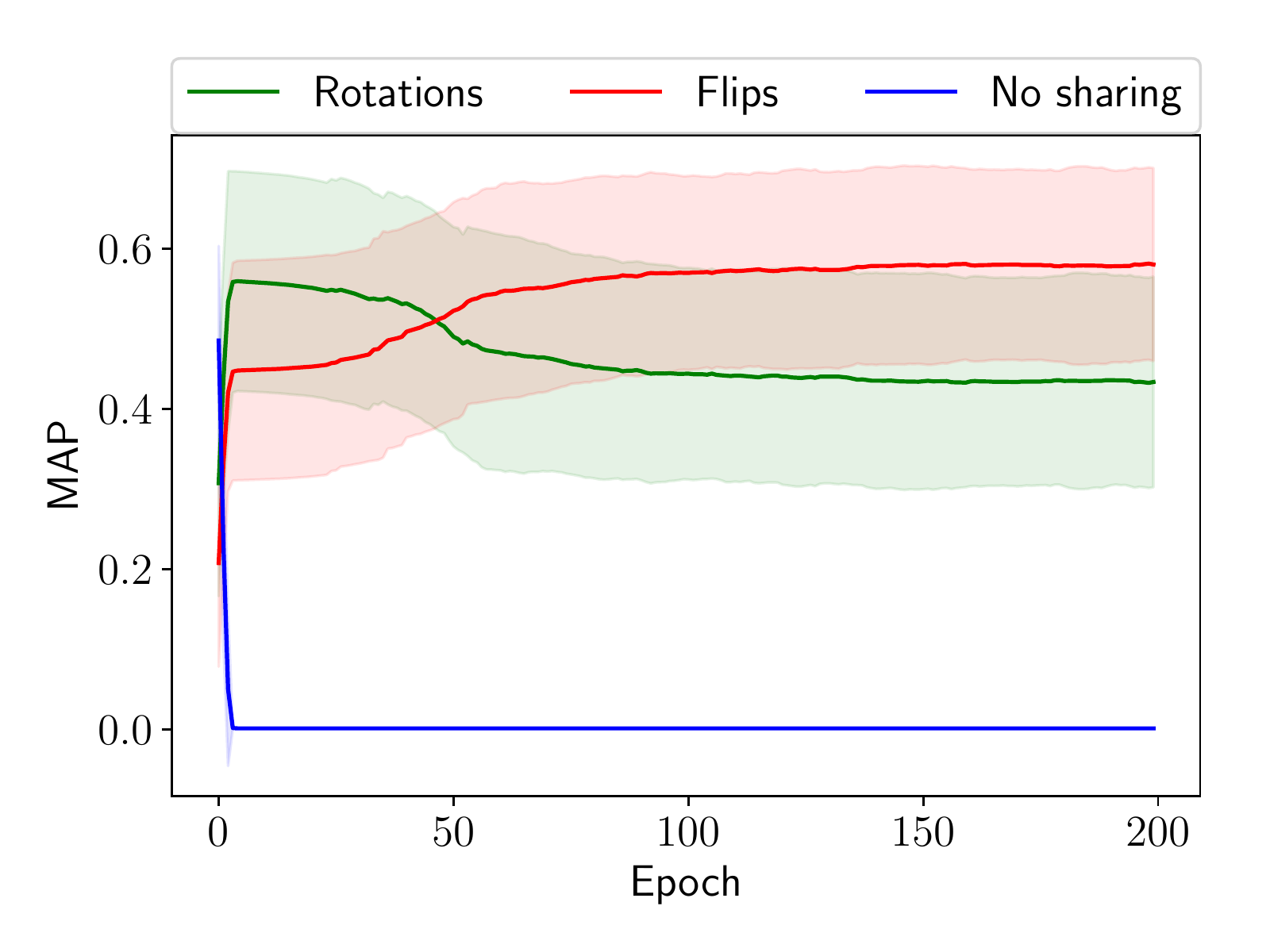}
         \caption{Normal dataset}
         \label{fig:maps-plain}
     \end{subfigure}
        \caption{Means of the MAP estimates for $\boldsymbol{\pi}$. We see that our network converges to the correct invariances in each case.}
        \label{fig:maps}
\end{figure*}

\section{Related Work}

\paragraph{Invariance in Neural Networks}
There have been many works that build invariant networks by finding a proper weight-sharing scheme. This is achieved by designing special kinds of layers, which make the resulting models invariant or equivariant to certain types of symmetries, like rotations and reflections \citep{pmlr-v48-cohenc16,ravanbakhsh2017equivariance,zhang2018learning}. As our goal was also to learn invariances, we were inspired by \citet{zhou2021metalearning}, \citet{mouli2021neural} and \citet{benton2020learning}. In the latter, distributions over invariances are parametrised and then encouraged via a regularisation term. This enables the simultaneous learning of the distributional and the network parameters through backpropagation. Contrary to our method, none of them employed the Bayesian paradigm.
\vspace{-5pt}

\paragraph{Bayesian neural networks}

Bayesian neural networks \citep{mackay1992practical, neal1992bayesian} have gained a lot of popularity recently \citep{blundell2015weight, hernandez2015probabilistic, wenzel2020good, fortuin2021bnnpriors, fortuin2021bayesian, fortuin2021priors, izmailov2021bayesian, immer2021scalable, immer2021improving, dangelo2021stein, dangelo2021repulsive}.
However, to the best of our knowledge, none of these works dealt explicitly with inferring posteriors over weight-sharing schemes to learn invariances from data.
\vspace{-5pt}

\paragraph{Probabilistic approaches to invariance}
Gaussian processes (GPs) are another probabilistic model that has been used to model invariant functions. In \citet{10.1007/978-3-319-00218-7_13}, it is shown that a GP has invariant trajectories if and only if its kernel is argument-wise invariant.
This specific idea is further developed in \citet{GINSBOURGER2016117} and \citet{AFST_2012_6_21_3_501_0}.
A similar method is developed in \citet{vdw2018invgp}, with the difference that there, invariances are incorporated into the structure of the model. This makes it possible to learn symmetries by maximizing the marginal likelihood of the model. These works mainly concern regression and low-dimensional tasks.
However, in \citet{vanderwilk2017convolutional}, a novel algorithm is developed that is suited for classification tasks with high-dimensional data like images.

\section{Experiments}
\label{sec:experiments}

\begin{figure*}[h]
     \centering
     
     \begin{subfigure}[b]{0.20\textwidth}
         \centering
         \includegraphics[width=\textwidth]{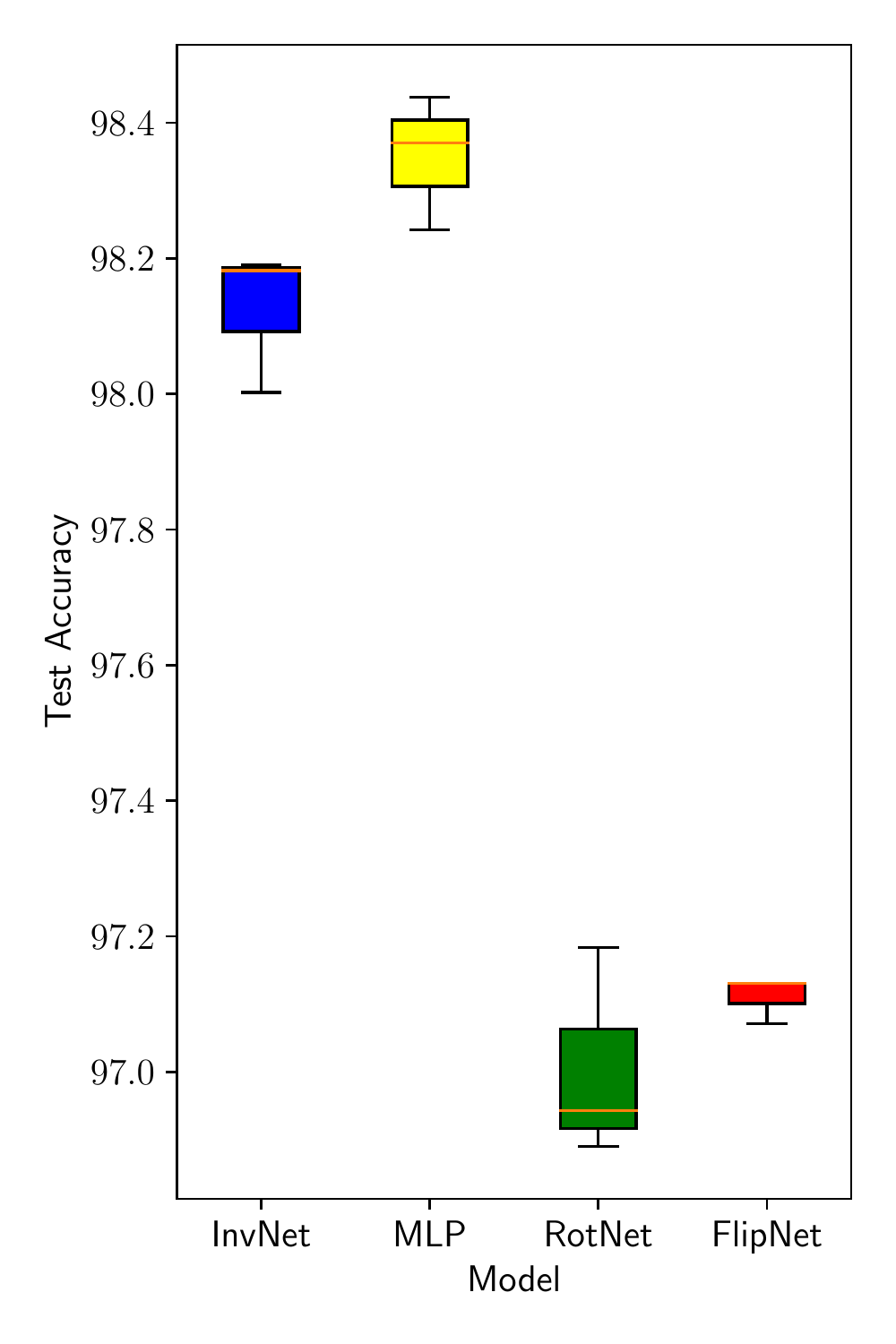}
         \caption{Permuted dataset}
         \label{fig:box-permuted}
     \end{subfigure}
     \begin{subfigure}[b]{0.20\textwidth}
         \centering
         \includegraphics[width=\textwidth]{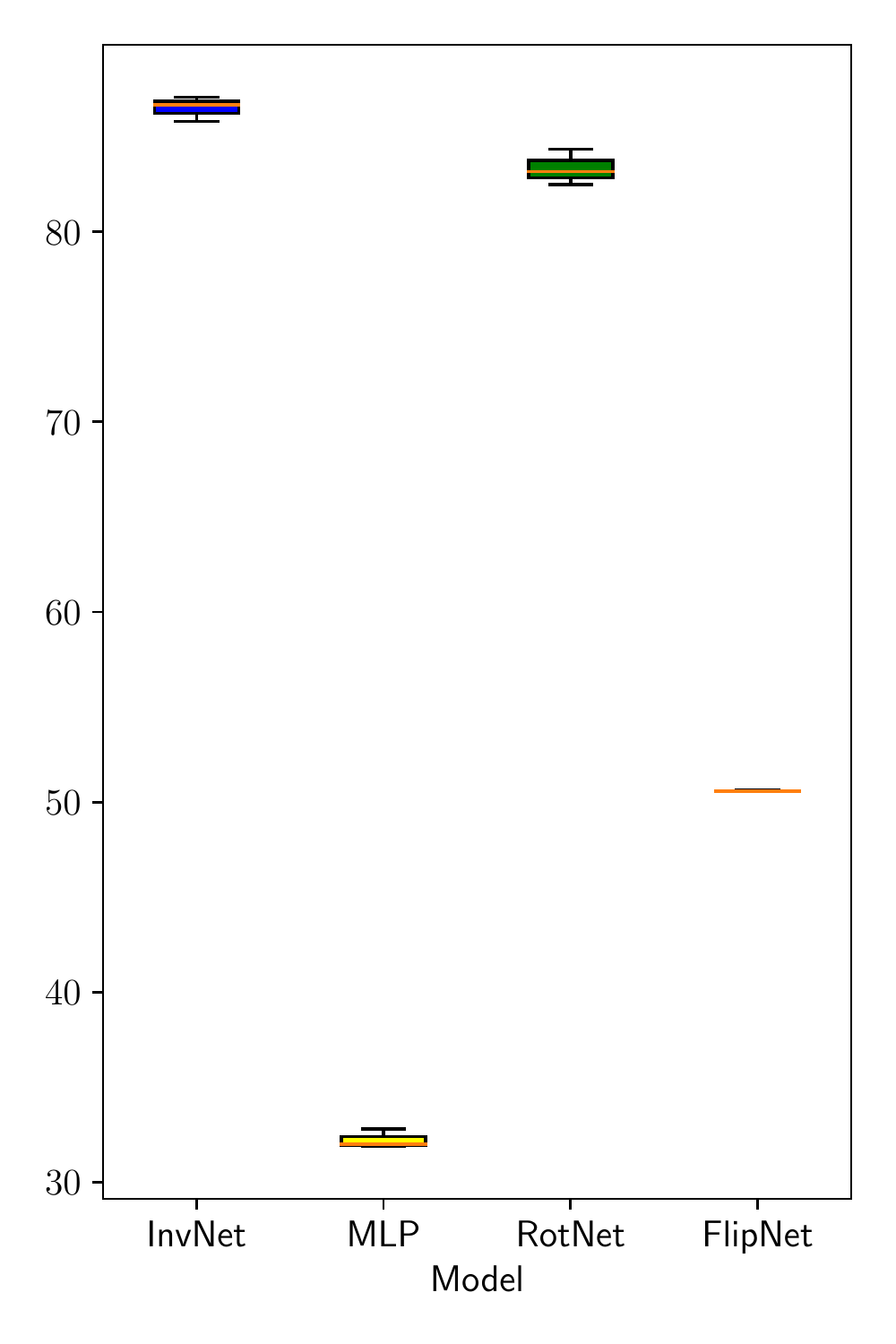}
         \caption{Rotated dataset}
         \label{fig:box_rotated}
     \end{subfigure}
    \begin{subfigure}[b]{0.20\textwidth}
         \centering
         \includegraphics[width=\textwidth]{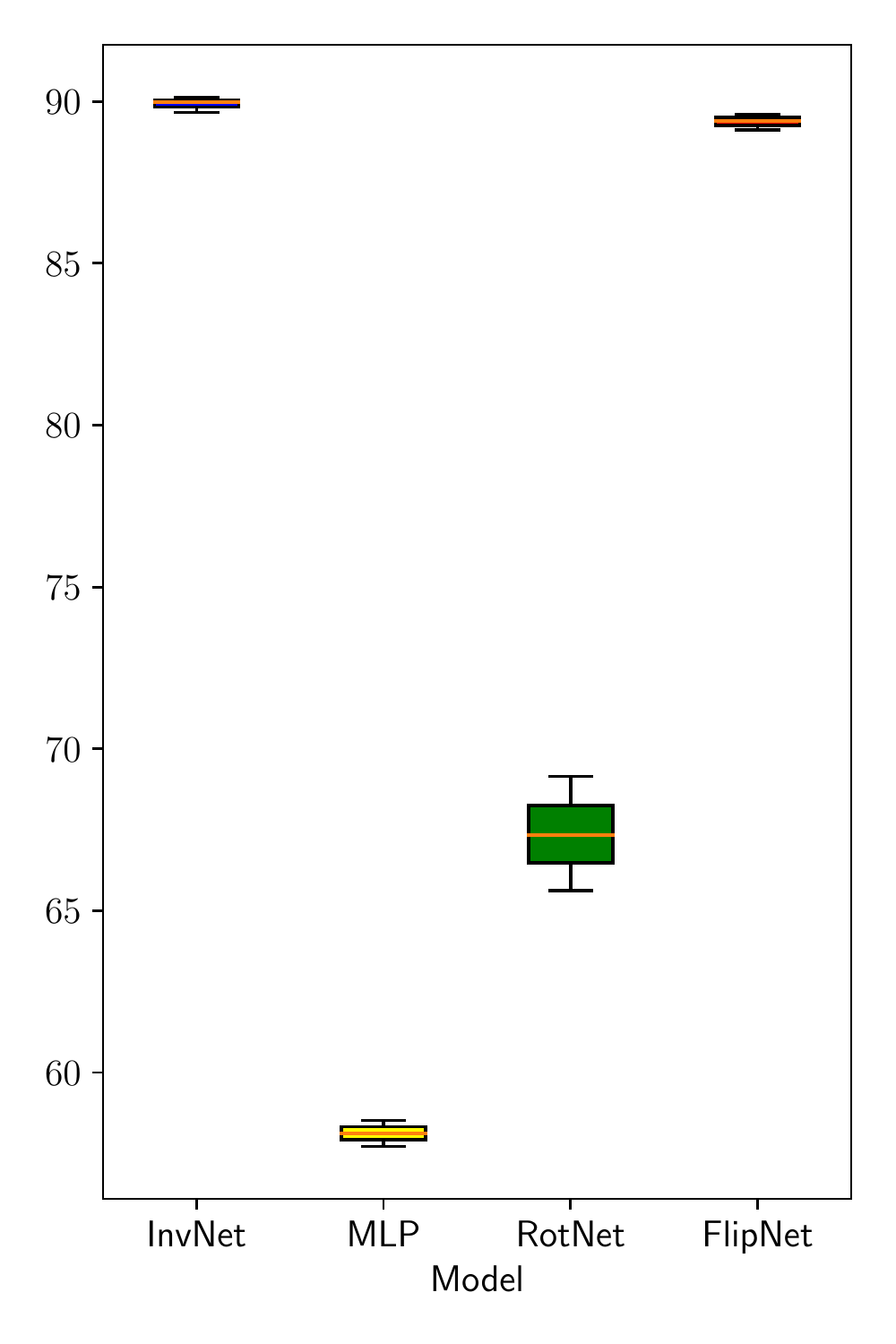}
         \caption{Flipped dataset}
         \label{fig:box-flipped}
     \end{subfigure}
        \caption{Test accuracies across different models and datasets. When symmetry is present in the data, our model performs competitively with the ``correct'' hand-engineered models, while discovering the respective invariances directly from the data.}
        \label{fig:boxplots}
\end{figure*}

We conducted experiments on the MNIST dataset \citep{lecun1998mnist}. Hence, our (flattened) input space is $\mathcal{X}=\mathbb{R}^{784}$, since the images are grayscale and of size $28\times28$. We focus on two groups of symmetries: Rotations of $k \times 90$ degrees, for $k \in \left\{0, 1, 2, 3 \right\}$, as well as vertical and horizontal flips.

We first carry out a sanity check for our model, by performing data augmentation.  Here, we assess whether our model can detect if these invariances are present in the data.  
Furthermore, we compare its test accuracy with that of other frequentist networks: two invariant architectures as well as one multi-layer perceptron (MLP).
Afterwards, we examine how ours and other non-invariant models (with respect to flips and rotations) perform when they are trained without data augmentation, but tested on transformed inputs.

\subsection{Validation of our model}

We used three different versions of MNIST. Firstly, one where all the pixels were randomly permuted according to a fixed permutation of size $784$. This allows the network to still recognize patterns in the images and perform its task, but also creates a situation where there is no symmetry present in the data. For the last two variations, we respectively flip and rotate some of the images. In both cases, every digit was transformed in exactly one way during training. This enables us to challenge the models regarding their extrapolating capacity to other configurations of the inputs. To this end, during testing, we randomly transform all the points. In addition, when it comes to rotations, all the images labeled with 6 or 9 were removed, as they are identical when rotated. An illustration of these datasets can be found in Figure~\ref{fig:digits}.

As mentioned above, we trained three models on these datasets: Ours, which we call InvariantNet, as well as its two frequentist analogues, that is, multi-layer perceptron (MLP) networks with the proper weight-sharing scheme between their input and first hidden layers.
These two are called FlipNet and RotNet. Lastly, we used a standard MLP for comparison.
We note that in this simple experiment, we are in the case where $q=1$. Hence, InvariantNet has two different versions, one for rotations and one for flips.

More details on the architectures of the networks and their training are available in Appendix~\ref{sec:imple}.

As expected, InvariantNet learns any invariance that was present in the data. This is confirmed by Figures~\ref{fig:maps-permuted}, ~\ref{fig:maps-rotated}, ~\ref{fig:maps-flipped}, where one can observe the trajectories of the MAP estimates of $\boldsymbol{\pi}$ during training for all three versions of MNIST.
Regarding the performances of the different models, which can be inspected in Figure~\ref{fig:boxplots}, we note that in the transformed cases, our model performed competitively with (and even slightly surpassed) the manually engineered invariant networks, namely RotNet and FlipNet. This is probably explained by the prior-regularization term. Moreover, we see that it was also second-best on the permuted dataset, closely approaching the performance of the MLP. 

\subsection{Out-of-distribution extrapolation}

\begin{table}[t]
\caption{Classification accuracies for different models trained on plain MNIST and tested on transformed data. Our model outperforms all the baselines on the transformed test sets.}

\label{table:acc}
\vskip 0.15in
\begin{center}
\begin{small}
\begin{sc}
\begin{tabular}{lcccr}
\toprule
Model & Plain & Rotated & Flipped \\
\midrule
InvNet    & 94.2 $\pm$ 0.2& \textbf{58.2 $\pm$ 1.3}&  \textbf{72.3 $\pm$ 1.5} \\
CNN    & \textbf{99.0 $\pm$ 0.2}& 47.2 $\pm$ 1.1& 63.9 $\pm$ 1.0 \\
MLP    & 98.3 $\pm$ 0.3& 43.3 $\pm$ 0.6& 64.1 $\pm$ 0.4  \\
BMLP    & 97.4 $\pm$ 0.2& 43.1 $\pm$ 0.9& 63.5 $\pm$ 0.8         \\
BCNN     & 98.6 $\pm$ 0.1& 47.0 $\pm$ 0.5& 63.1 $\pm$ 0.4 \\

\bottomrule
\end{tabular}
\end{sc}
\end{small}
\end{center}
\vskip -0.1in
\end{table}

We now focus on the normal MNIST, where we did not synthetically build in any symmetries and study the case $q=2$, that is, we consider being invariant to both groups simultaneously. Again, by making an uninformative choice for the hyperparameter of the Dirichlet distribution, we encourage the model to take into account all possible choices, namely being invariant to rotations, flips, or having no sharing at all.  According to Figure~\ref{fig:maps-plain}, the model believes that it is almost equally probable that rotations or flips could be present. 
To test whether this would lead to better generalization, we trained our model and three baselines on the plain dataset without augmentation.
Specifically, we chose frequentist and Bayesian MLPs and CNNs \cite{bcnn}. For the Bayesian models, we used a standard Gaussian prior and approximated the posterior distribution with Bayes-by-Backprop, using the local reparametrization trick \cite{kingma2015variational}.

We see in Table~\ref{table:acc} that InvariantNet clearly outperforms all the non-invariant models when it is asked to infer the classes of transformed examples. This comes at the price of a slightly lower accuracy on plain images, which nonetheless remains quite competitive. This is because the network sacrifices some of its flexibility, which would make it perform better on normal examples, in order to be invariant to flips and rotations during the inference stage.

\section{Conclusion}

In this paper, we presented a novel probabilistic approach to performing (approximately) invariant classification. Our method is general and can be applied to an arbitrary number of groups, as long as matrix representations of the desired symmetries are available.
The proposed model is able to learn invariances of interest from augmented data, as well as to generalize well to examples that were not included during the augmentation.
Moreover, we have shown that our network extrapolates well to out-of-distribution samples when trained without data augmentation, compared to other non-invariant models.

An avenue for future work could be to move from MAP estimation to approximation of the full posterior distribution over the parameters of the Categorical distribution. What is more, to obtain uncertainty estimates of the predictions, one could treat all the weights of the network stochastically, thus turning it into a fully Bayesian neural network. Finally, one could try to be completely agnostic and also learn the weight-sharing scheme from the data, by defining a proper priors over the symmetry matrices. 

\bibliographystyle{icml2021}
\bibliography{references}

\begin{thebibliography}{33}
\providecommand{\natexlab}[1]{#1}
\providecommand{\url}[1]{\texttt{#1}}
\expandafter\ifx\csname urlstyle\endcsname\relax
  \providecommand{\doi}[1]{doi: #1}\else
  \providecommand{\doi}{doi: \begingroup \urlstyle{rm}\Url}\fi

\bibitem[Benton et~al.(2020)Benton, Finzi, Izmailov, and
  Wilson]{benton2020learning}
Benton, G., Finzi, M., Izmailov, P., and Wilson, A.~G.
\newblock Learning invariances in neural networks, 2020.

\bibitem[Blundell et~al.(2015)Blundell, Cornebise, Kavukcuoglu, and
  Wierstra]{blundell2015weight}
Blundell, C., Cornebise, J., Kavukcuoglu, K., and Wierstra, D.
\newblock Weight uncertainty in neural networks.
\newblock \emph{arXiv preprint arXiv:1505.05424}, 2015.

\bibitem[Cohen \& Welling(2016)Cohen and Welling]{pmlr-v48-cohenc16}
Cohen, T. and Welling, M.
\newblock Group equivariant convolutional networks.
\newblock In Balcan, M.~F. and Weinberger, K.~Q. (eds.), \emph{Proceedings of
  The 33rd International Conference on Machine Learning}, volume~48 of
  \emph{Proceedings of Machine Learning Research}, pp.\  2990--2999, New York,
  New York, USA, 20--22 Jun 2016. PMLR.
\newblock URL \url{http://proceedings.mlr.press/v48/cohenc16.html}.

\bibitem[D'Angelo \& Fortuin(2021)D'Angelo and Fortuin]{dangelo2021repulsive}
D'Angelo, F. and Fortuin, V.
\newblock Repulsive deep ensembles are bayesian.
\newblock \emph{arXiv preprint arXiv:2106.11642}, 2021.

\bibitem[D'Angelo et~al.(2021)D'Angelo, Fortuin, and Wenzel]{dangelo2021stein}
D'Angelo, F., Fortuin, V., and Wenzel, F.
\newblock On stein variational neural network ensembles.
\newblock \emph{arXiv preprint arXiv:2106.10760}, 2021.

\bibitem[Fortuin(2021)]{fortuin2021priors}
Fortuin, V.
\newblock Priors in bayesian deep learning: A review.
\newblock \emph{arXiv preprint arXiv:2105.06868}, 2021.

\bibitem[Fortuin et~al.(2021{\natexlab{a}})Fortuin, Garriga-Alonso, van~der
  Wilk, and Aitchison]{fortuin2021bnnpriors}
Fortuin, V., Garriga-Alonso, A., van~der Wilk, M., and Aitchison, L.
\newblock Bnnpriors: A library for bayesian neural network inference with
  different prior distributions.
\newblock \emph{Software Impacts}, pp.\  100079, 2021{\natexlab{a}}.

\bibitem[Fortuin et~al.(2021{\natexlab{b}})Fortuin, Garriga-Alonso, Wenzel,
  R{\"a}tsch, Turner, van~der Wilk, and Aitchison]{fortuin2021bayesian}
Fortuin, V., Garriga-Alonso, A., Wenzel, F., R{\"a}tsch, G., Turner, R.,
  van~der Wilk, M., and Aitchison, L.
\newblock Bayesian neural network priors revisited.
\newblock \emph{arXiv preprint arXiv:2102.06571}, 2021{\natexlab{b}}.

\bibitem[Ginsbourger et~al.(2012)Ginsbourger, Bay, Roustant, and
  Carraro]{AFST_2012_6_21_3_501_0}
Ginsbourger, D., Bay, X., Roustant, O., and Carraro, L.
\newblock Argumentwise invariant kernels for the approximation of invariant
  functions.
\newblock \emph{Annales de la Facult\'e des sciences de Toulouse :
  Math\'ematiques}, Ser. 6, 21\penalty0 (3):\penalty0 501--527, 2012.
\newblock \doi{10.5802/afst.1343}.
\newblock URL
  \url{https://afst.centre-mersenne.org/articles/10.5802/afst.1343/}.

\bibitem[Ginsbourger et~al.(2013)Ginsbourger, Durrande, and
  Roustant]{10.1007/978-3-319-00218-7_13}
Ginsbourger, D., Durrande, N., and Roustant, O.
\newblock Kernels and designs for modelling invariant functions: From group
  invariance to additivity.
\newblock In Ucinski, D., Atkinson, A.~C., and Patan, M. (eds.), \emph{mODa 10
  -- Advances in Model-Oriented Design and Analysis}, pp.\  107--115,
  Heidelberg, 2013. Springer International Publishing.
\newblock ISBN 978-3-319-00218-7.

\bibitem[Ginsbourger et~al.(2016)Ginsbourger, Roustant, and
  Durrande]{GINSBOURGER2016117}
Ginsbourger, D., Roustant, O., and Durrande, N.
\newblock On degeneracy and invariances of random fields paths with
  applications in gaussian process modelling.
\newblock \emph{Journal of Statistical Planning and Inference}, 170:\penalty0
  117--128, 2016.
\newblock \doi{https://doi.org/10.1016/j.jspi.2015.10.002}.
\newblock URL
  \url{https://www.sciencedirect.com/science/article/pii/S0378375815001640}.

\bibitem[Hern{\'a}ndez-Lobato \& Adams(2015)Hern{\'a}ndez-Lobato and
  Adams]{hernandez2015probabilistic}
Hern{\'a}ndez-Lobato, J.~M. and Adams, R.
\newblock Probabilistic backpropagation for scalable learning of bayesian
  neural networks.
\newblock In \emph{International Conference on Machine Learning}, pp.\
  1861--1869. PMLR, 2015.

\bibitem[Immer et~al.(2021{\natexlab{a}})Immer, Bauer, Fortuin, R{\"a}tsch, and
  Khan]{immer2021scalable}
Immer, A., Bauer, M., Fortuin, V., R{\"a}tsch, G., and Khan, M.~E.
\newblock Scalable marginal likelihood estimation for model selection in deep
  learning.
\newblock \emph{arXiv preprint arXiv:2104.04975}, 2021{\natexlab{a}}.

\bibitem[Immer et~al.(2021{\natexlab{b}})Immer, Korzepa, and
  Bauer]{immer2021improving}
Immer, A., Korzepa, M., and Bauer, M.
\newblock Improving predictions of bayesian neural nets via local
  linearization.
\newblock In \emph{International Conference on Artificial Intelligence and
  Statistics}, pp.\  703--711. PMLR, 2021{\natexlab{b}}.

\bibitem[Izmailov et~al.(2021)Izmailov, Vikram, Hoffman, and
  Wilson]{izmailov2021bayesian}
Izmailov, P., Vikram, S., Hoffman, M.~D., and Wilson, A.~G.
\newblock What are bayesian neural network posteriors really like?
\newblock \emph{arXiv preprint arXiv:2104.14421}, 2021.

\bibitem[Jang et~al.(2017)Jang, Gu, and Poole]{jang2017categorical}
Jang, E., Gu, S., and Poole, B.
\newblock Categorical reparameterization with gumbel-softmax, 2017.

\bibitem[Jeffreys(1946)]{10.2307/97883}
Jeffreys, H.
\newblock An invariant form for the prior probability in estimation problems.
\newblock \emph{Proceedings of the Royal Society of London. Series A,
  Mathematical and Physical Sciences}, 186\penalty0 (1007):\penalty0 453--461,
  1946.
\newblock ISSN 00804630.
\newblock URL \url{http://www.jstor.org/stable/97883}.

\bibitem[Kingma et~al.(2015)Kingma, Salimans, and
  Welling]{kingma2015variational}
Kingma, D.~P., Salimans, T., and Welling, M.
\newblock Variational dropout and the local reparameterization trick, 2015.

\bibitem[Kondor \& Trivedi(2018)Kondor and Trivedi]{pmlr-v80-kondor18a}
Kondor, R. and Trivedi, S.
\newblock On the generalization of equivariance and convolution in neural
  networks to the action of compact groups.
\newblock In Dy, J. and Krause, A. (eds.), \emph{Proceedings of the 35th
  International Conference on Machine Learning}, volume~80 of \emph{Proceedings
  of Machine Learning Research}, pp.\  2747--2755. PMLR, 10--15 Jul 2018.
\newblock URL \url{http://proceedings.mlr.press/v80/kondor18a.html}.

\bibitem[LeCun(1998)]{lecun1998mnist}
LeCun, Y.
\newblock The mnist database of handwritten digits.
\newblock \emph{http://yann. lecun. com/exdb/mnist/}, 1998.

\bibitem[Lyle et~al.(2020)Lyle, van~der Wilk, Kwiatkowska, Gal, and
  Bloem-Reddy]{lyle2020benefits}
Lyle, C., van~der Wilk, M., Kwiatkowska, M., Gal, Y., and Bloem-Reddy, B.
\newblock On the benefits of invariance in neural networks, 2020.
\newblock URL \url{https://arxiv.org/abs/2005.00178}.

\bibitem[MacKay(1992)]{mackay1992practical}
MacKay, D.~J.
\newblock A practical {B}ayesian framework for backpropagation networks.
\newblock \emph{Neural computation}, 4\penalty0 (3):\penalty0 448--472, 1992.

\bibitem[Mouli \& Ribeiro(2021)Mouli and Ribeiro]{mouli2021neural}
Mouli, S.~C. and Ribeiro, B.
\newblock Neural networks for learning counterfactual g-invariances from single
  environments.
\newblock In \emph{International Conference on Learning Representations}, 2021.
\newblock URL \url{https://openreview.net/forum?id=7t1FcJUWhi3}.

\bibitem[Neal(1992)]{neal1992bayesian}
Neal, R.~M.
\newblock {B}ayesian training of backpropagation networks by the {H}ybrid
  {M}onte {C}arlo method.
\newblock Technical report, University of Toronto, 1992.

\bibitem[Ravanbakhsh et~al.(2017)Ravanbakhsh, Schneider, and
  Poczos]{ravanbakhsh2017equivariance}
Ravanbakhsh, S., Schneider, J., and Poczos, B.
\newblock Equivariance through parameter-sharing, 2017.

\bibitem[Shridhar et~al.(2019)Shridhar, Laumann, and Liwicki]{bcnn}
Shridhar, K., Laumann, F., and Liwicki, M.
\newblock A comprehensive guide to bayesian convolutional neural network with
  variational inference, 2019.

\bibitem[van~der Pol et~al.(2021)van~der Pol, Worrall, van Hoof, Oliehoek, and
  Welling]{vanderpol2021mdp}
van~der Pol, E., Worrall, D.~E., van Hoof, H., Oliehoek, F.~A., and Welling, M.
\newblock Mdp homomorphic networks: Group symmetries in reinforcement learning,
  2021.

\bibitem[van~der Wilk et~al.(2017)van~der Wilk, Rasmussen, and
  Hensman]{vanderwilk2017convolutional}
van~der Wilk, M., Rasmussen, C.~E., and Hensman, J.
\newblock Convolutional gaussian processes, 2017.

\bibitem[van~der Wilk et~al.(2018)van~der Wilk, Bauer, John, and
  Hensman]{vdw2018invgp}
van~der Wilk, M., Bauer, M., John, S., and Hensman, J.
\newblock Learning invariances using the marginal likelihood.
\newblock In Bengio, S., Wallach, H., Larochelle, H., Grauman, K.,
  Cesa-Bianchi, N., and Garnett, R. (eds.), \emph{Advances in Neural
  Information Processing Systems 31 (NeurIPS)}, pp.\  9960--9970. Curran
  Associates, Inc., 2018.

\bibitem[Wenzel et~al.(2020)Wenzel, Roth, Veeling, {\'S}wi{\,{a}}tkowski, Tran,
  Mandt, Snoek, Salimans, Jenatton, and Nowozin]{wenzel2020good}
Wenzel, F., Roth, K., Veeling, B.~S., {\'S}wi{\,{a}}tkowski, J., Tran, L.,
  Mandt, S., Snoek, J., Salimans, T., Jenatton, R., and Nowozin, S.
\newblock How good is the {B}ayes posterior in deep neural networks really?
\newblock In \emph{International Conference on Machine Learning}, 2020.

\bibitem[Yarotsky(2018)]{yarotsky2018universal}
Yarotsky, D.
\newblock Universal approximations of invariant maps by neural networks, 2018.

\bibitem[Zhang et~al.(2018)Zhang, Wang, Figueiredo, and
  Balzano]{zhang2018learning}
Zhang, D., Wang, H., Figueiredo, M., and Balzano, L.
\newblock Learning to share: Simultaneous parameter tying and sparsification in
  deep learning,.
\newblock In \emph{International Conference on Learning Representations}, 2018.
\newblock URL \url{https://openreview.net/forum?id=rypT3fb0b}.

\bibitem[Zhou et~al.(2021)Zhou, Knowles, and Finn]{zhou2021metalearning}
Zhou, A., Knowles, T., and Finn, C.
\newblock Meta-learning symmetries by reparameterization.
\newblock In \emph{International Conference on Learning Representations}, 2021.
\newblock URL \url{https://openreview.net/forum?id=-QxT4mJdijq}.

\end{thebibliography}

\clearpage
\appendix

\begin{figure*}[h!]
     \centering
     \begin{subfigure}[b]{0.2\textwidth}
         \centering
         \includegraphics[width=\textwidth]{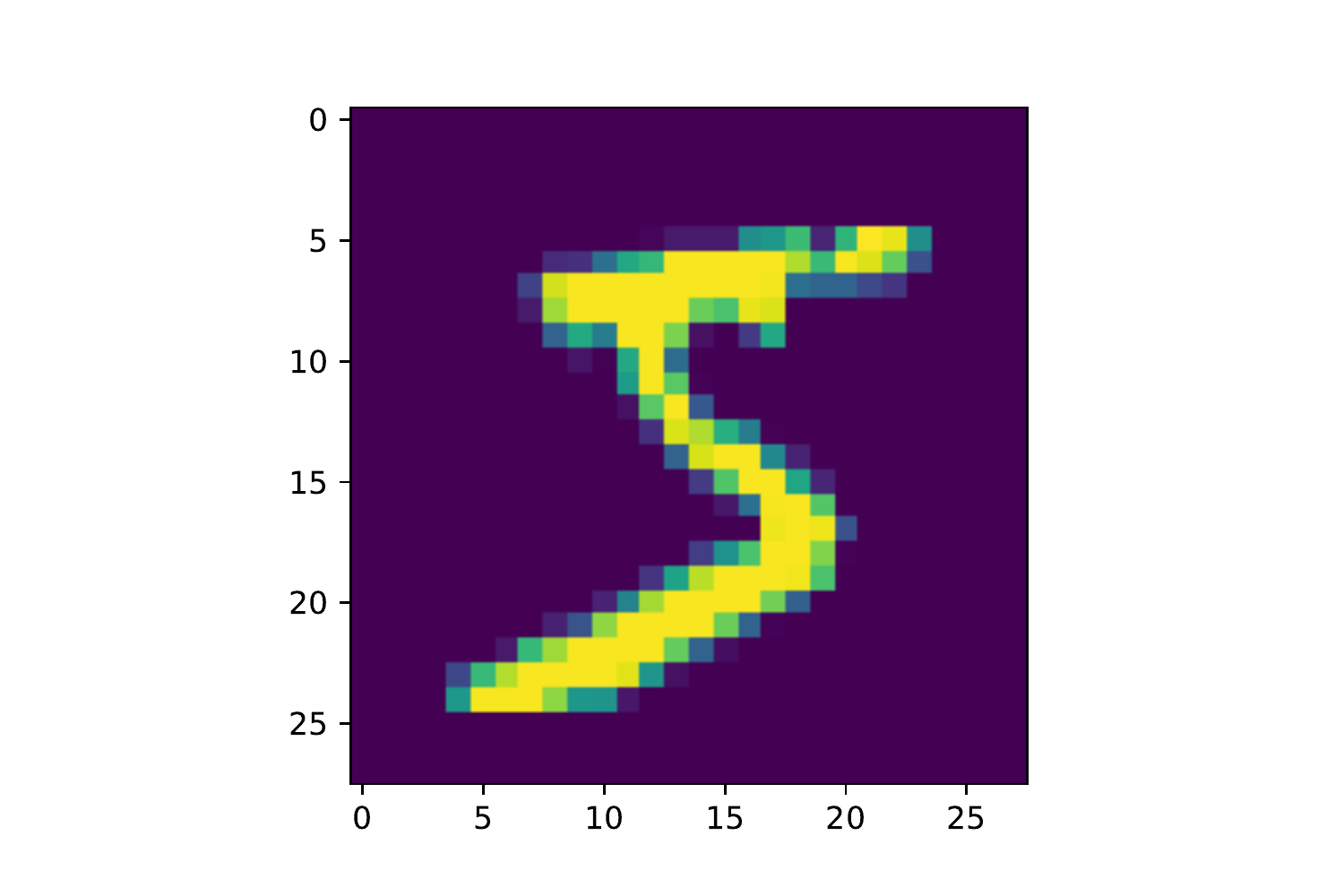}
         \caption{Normal dataset}
         \label{fig:digits-plain}
     \end{subfigure}
     \begin{subfigure}[b]{0.2\textwidth}
         \centering
         \includegraphics[width=\textwidth]{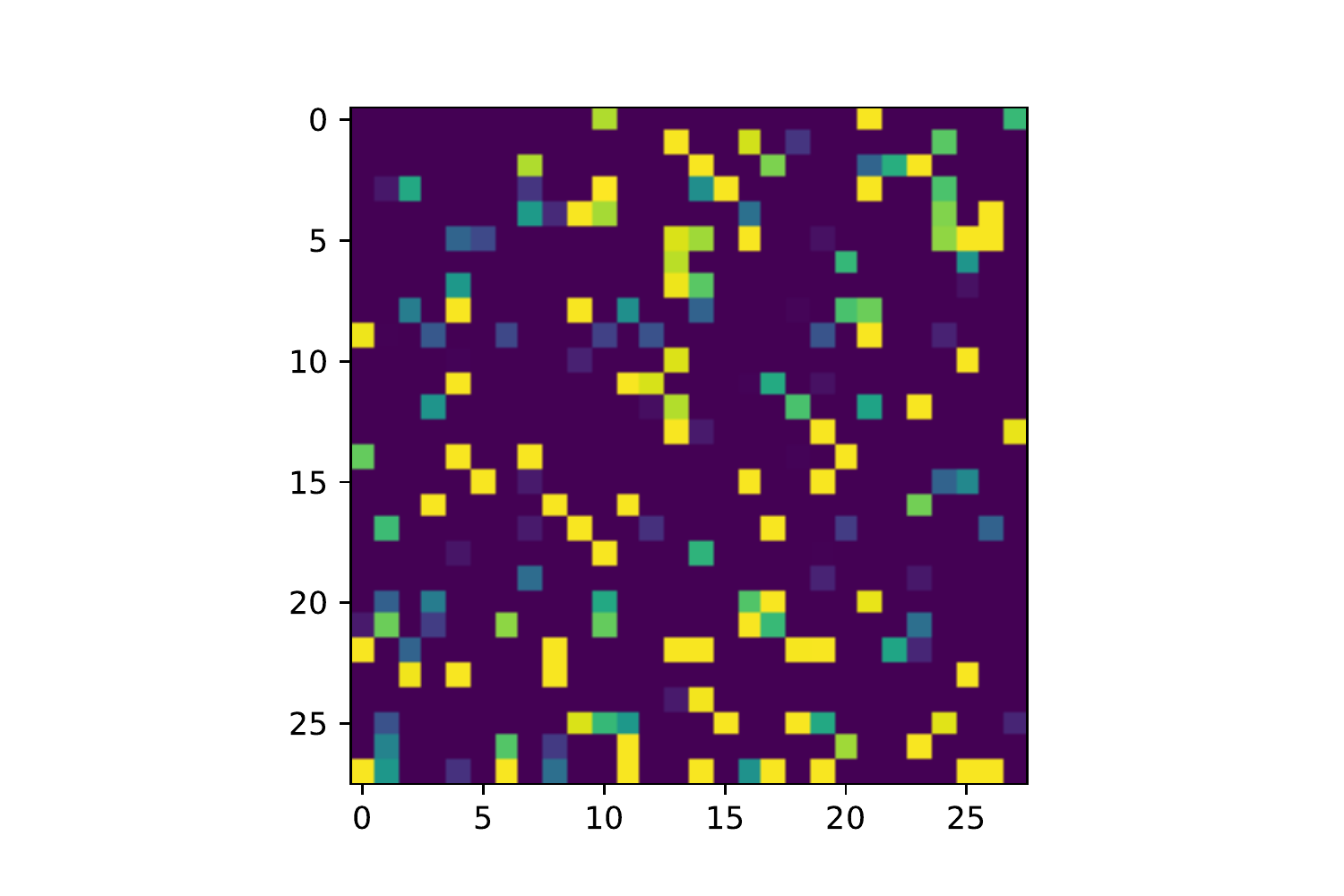}
         \caption{Permuted dataset}
         \label{fig:digits-permuted}
     \end{subfigure}
     \begin{subfigure}[b]{0.2\textwidth}
         \centering
         \includegraphics[width=\textwidth]{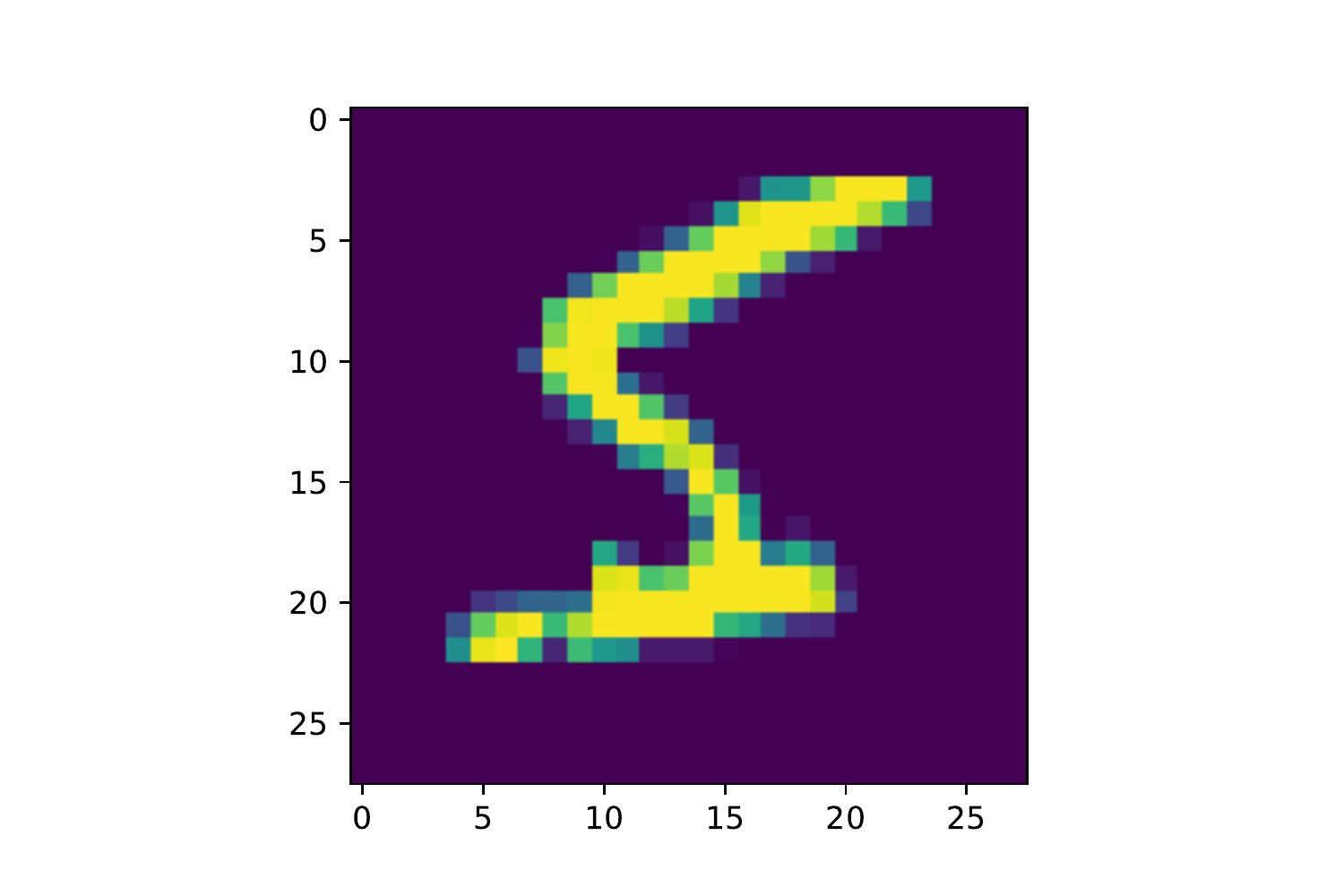}
         \caption{Rotated dataset}
         \label{fig:digits-rotated}
     \end{subfigure}
    \begin{subfigure}[b]{0.2\textwidth}
         \centering
         \includegraphics[width=\textwidth]{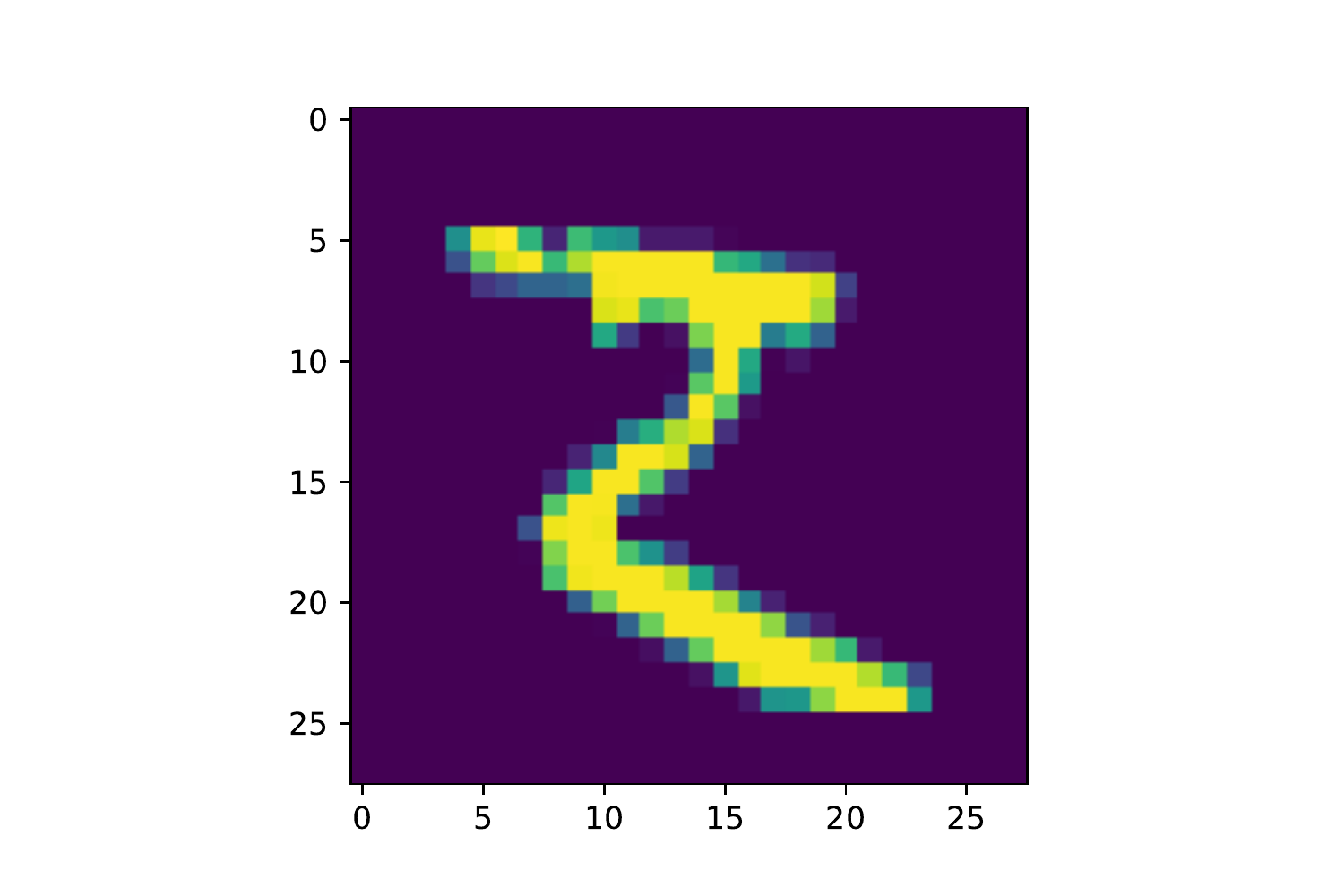}
         \caption{Flipped dataset}
         \label{fig:digits-flipped}
     \end{subfigure}
        \caption{Examples for the four datasets.}
        \label{fig:digits}
\end{figure*}

\section{Derivations}
\label{sec:algebra}

We first introduce in more detail some algebraic concepts that are used in our method. Then, we will see how one can construct a weight-sharing scheme with respect to a given group.  We follow \citet{mouli2021neural}.

Let $G$ be a finite group of linear automorphisms over $\mathbb{R}^d$. As a reminder, an automorphism is a linear transformation from $\mathbb{R}^d$ to itself, which is also a bijection.

We now prove the following two lemmas regarding the invariance of the Reynolds operator and the construction of invariant neurons.
\begin{lemma}
Suppose $\overline{T}$ is the Reynolds operator of the group $G$. Then $\overline{T}$ is $G$-invariant, that is, it holds that: $\overline{T}(T x)=\overline{T}(x)$,  for all $x \in \mathbb{R}^d$ and $T \in G$. Furthermore, $\overline{T}$ is a projection, that is, $\overline{T}^2=T$.
\end{lemma}
\begin{proof}
Fix $F \in G$. Then we compute:
$$
\overline{T} \circ F=\frac{1}{|G|} \sum_{T \in G} T \circ F=\frac{1}{|G|} \sum_{T^{\prime} \in G^{\prime}} T^{\prime},
$$

where $G^{\prime}=\left\{T \circ F: \forall T \in G\right\}$. If we had that $G^{\prime}=G$, then we would be done. Indeed, as groups are closed under multiplication, it holds that $T \circ F \in G$ for an arbitrary $T \in G$ and as a consequence,  $G^{\prime} \subseteq  G$. Now, as $F$ is a bijection, by multiplying with $F^{-1}$ from the right, we get that $T_1 \circ F=T_2 \circ F$ if and only if $T_1=T_2$, where $T_1, T_2$ are arbitrary elements of $G$. Thus, $|G|=|G^{\prime}|$ and the result follows.

Now, we will show that $\overline{T}$ is a projection. Observe that:
$$\overline{T}^2= \overline{T}\left (\frac{1}{|G|}\sum_{T \in G}T\right) = \frac{1}{|G|}\sum_{T \in G} \overline{T}\circ T=\overline{T},$$
where in the second equality we used the linearity of $\overline{T}$ and its $G$-invariance. 
\end{proof}

It is easy to verify that the eigenvalues of a projection are either 0 or 1. To see that, fix a non-zero eigenvector $w$ of $\overline{T}$. Then, $w^\top\overline{T}^2= (w^\top \overline{T})\overline{T}= \lambda w^\top \overline{T}=\lambda^2 w^\top$. On the other hand, $w^\top \overline{T}^2=w^\top \overline{T}= \lambda w^\top$. Hence, we must have that $\lambda=0$ or $1$.

Let $\Lambda = \left\{ w \in \mathbb{R}^d: w^\top \overline{T}= w^\top \right\}$ be the eigenspace of $\overline{T}$, which corresponds to the eigenvalue 1. Then we have the following result that characterizes $G$-invariant neurons:

\begin{lemma}
Let $b \in \mathbb{R}$, $w \in \mathbb{R}^d$ and $f(x)=w^\top x+b$, for $x \in \mathbb{R}^d$. Then, for an arbitrary $T \in G$, we have $f(Tx)=f(x)$ if and only if $w \in \Lambda$. 
\end{lemma}
\begin{proof}
We only prove the first direction. Consider the orthogonal basis $(w_i)_{i=1}^m$ of $\Lambda$, which is constituted of the eigenvectors that correspond to the eigenvalue 1.

Fix $0 \neq w \in \Lambda$. Then, we can find $(c_i)_{i=1}^m \subseteq \mathbb{R}$, such that $w^\top= \sum_{i=1}^{m} c_i w_{i}^\top= \sum_{i=1}^m c_i w_{i}^\top \overline{T}$, since $w_{i}^\top \overline{T}=w_{i}^\top$ for all $i=1,...,m$. 
Now, for $x \in \mathbb{R}^d$ and $T \in G$, we have that:
$$
\begin{aligned}
f(Tx) &=w^\top (Tx)+b \\
&=\sum_{i=1}^{m}  w_{i}^{T} c_{i}  \overline{T}(Tx)+b \\
&=\sum_{i=1}^{m}  w_{i}^{T} c_{i} \overline{T} x+b \\
&= w^\top \overline{T}x+b \\
&= w^\top x+b \\
&= f(x),
\end{aligned}
$$
where in the third equality we used the invariance of $\overline{T}$ and in the fifth one that $w \in \Lambda$. 
\end{proof}

Because of the above, one way to construct a $G$-invariant neuron is to constraint it to lie in $\Lambda$. Hence, to do so, we proceed as follows: We compute $\overline{T}$; we find the basis of $\Lambda$; finally, we learn the coefficients $(c_i)_{i=1}^m$.

We give a specific example on how to compute the Reynolds operator for the the following group of rotations: $G= \left \{ Id, T_{90}, T_{180}, T_{270} \right \}$. For simplicity, we consider that we have inputs of size $2\times2$, and consequently our (flattened) input domain is $\mathcal{X}=\mathbb{R}^4$. Constructions for inputs of higher dimensions are similar. The corresponding matrix representations of $G$ are the following:

$$\left \{ 
\left[\begin{array}{llll}
1 & 0 & 0 & 0\\
0 & 1 & 0 & 0\\
0 & 0 & 1 & 0 \\
0 & 0 & 0 & 1
\end{array}\right], 
\left[\begin{array}{llll}
0 & 1 & 0 & 0\\
0 & 0 & 0 & 1\\
1 & 0 & 0 & 0 \\
0 & 0 & 1 & 0
\end{array}\right], 
\right.$$ \\
$$\left.
\left[\begin{array}{llll}
0 & 0 & 0 & 1\\
0 & 0 & 1 & 0\\
0 & 1 & 0 & 0 \\
1 & 0 & 0 & 0
\end{array}\right], 
\left[\begin{array}{llll}
0 & 0 & 1 & 0\\
1 & 0 & 0 & 0\\
0 & 0 & 0 & 1 \\
0 & 1 & 0 &0
\end{array}
\right] 
\right \}$$

If we sum these matrices and divide by the cardinality of $G$, which is 4, we get the representation for the Reynolds operator:
$\overline{T}=\frac{1}{4} \mathbf{I}_{4 \times 4}$,  where $ \mathbf{I}_{4 \times 4}$ is the $4\times 4$ square matrix whose entries are all $1$. One can easily check that it is $G$-invariant, that is, it holds that $\overline{T}T_{k}=\overline{T}$, for $k  \in \left \{0, 90, 180, 270 \right \}$. Lastly, in this case, $\Lambda=span\left \{ (1, 1, 1, 1)^\top \right \}$, hence the parameter-sharing scheme is simply $V=(1,1,1, 1)^\top$.

\section{Implementation Details}
\label{sec:imple}

For the first experiment, we used three variants of the MNIST dataset \cite{lecun1998mnist}: One where we permuted all the pixels, and two augmented versions with flips and rotations. The digits 6 and 9 were removed for rotations. Examples can be found in Figure~\ref{fig:digits}. 

All neural networks architectures had one hidden layer of width 100 and were trained until convergence, using early stopping. 

Finally, for the second experiment, we gave more capacity to InvariantNet, by using a hidden layer of size 200. For its training, we used $\operatorname{Dir}(2,2,2)$ as a prior and temperature 1.

\end{document}